%% file: paper-arxiv.tex
\pdfoutput=1

\documentclass{article}

\usepackage{times, graphicx, subfigure, natbib}

\usepackage{easybmat}
\usepackage{amsmath, amsthm, amssymb}
\usepackage[ruled,vlined,linesnumbered,inoutnumbered]{algorithm2e}

\newcommand{\useFIG}[1]{#1} 
\newcommand{\Reg}{R} 
\newcommand{\Loss}{L} 

\SetKwFor{ParallelForEach}{for each}{in parallel do}{endfor}

\usepackage[accepted]{icml2014a}

\icmltitlerunning{Distributed Coordinate Descent Method for Learning with Big Data}

\include{commandsICML2014}

\begin{document}

\twocolumn[
\icmltitle{Distributed Coordinate Descent Method for Learning with Big Data}


\icmlauthor{Peter Richt\'{a}rik}{peter.richtarik@ed.ac.uk}
\icmlauthor{Martin Tak\'{a}\v{c}}{martin.taki@gmail.com}
\icmladdress{University of Edinburgh, King's Buildings, EH9 3JZ Edinburgh, United Kingdom}

\icmlkeywords{coordinate descent, distributed algorithms, boosting}

\vskip 0.3in
]

\begin{abstract}
In this paper we develop and analyze Hydra:  HYbriD cooRdinAte descent method for solving loss minimization problems with big data. We initially partition the coordinates (features) and assign each partition to a different node  of a cluster. At every iteration, each node picks a random subset of the coordinates from those it owns, independently from the other computers, and in parallel computes and applies updates to the selected coordinates based on a simple closed-form formula. We give bounds on the number of iterations sufficient to approximately solve the problem with high probability, and show how it depends on the data and on the partitioning. We perform numerical experiments with a LASSO instance described by a 3TB matrix.
\end{abstract}

\section{Introduction}

Randomized coordinate descent methods (CDMs) are increasingly popular in many learning tasks, including boosting, large scale regression and training linear support vector machines. CDMs update a single randomly chosen coordinate at a time by moving in the direction of the negative partial derivative (for smooth losses). Methods of this type, in various settings, were studied by several authors, including \citet{SDCA-2008, SSS-TewariICML09, Nesterov:2010RCDM, RT:UCDC, Necoara:Coupled, TRG:InexactCDM, Stoch-dual-Coord-Ascent, lu2013complexity}.

It is clear that in order to utilize modern shared-memory parallel computers, more coordinates should be updated at each iteration. One way to approach this is via partitioning the coordinates into blocks, and operating on a single randomly chosen block at a time, utilizing parallel linear algebra libraries. This approach was pioneered by \citet{Nesterov:2010RCDM} for smooth losses, and was extended to regularized problems in \citep{RT:UCDC}. Another popular approach involves working with a random subset of coordinates \citep{Bradley:PCD-paper}. These approaches can be combined, and theory was developed for methods that update a random subset of blocks of coordinates at a time \cite{RT:PCDM, FR:SPCDM2013}. Further recent works on parallel coordinate descent  include \citep{RT:TTD2011, BoostingMomentum2013,  Fer-ParallelAdaboost, TRB2013:DQA,SSS2013-accelerated}.

However, none of these methods are  directly scalable to problems of sizes so large that a single computer is unable to store the data describing the instance, or is unable to do so efficiently (e.g., in memory). In a big data scenario of this type, it is imperative to split the data across several nodes (computers) of a cluster, and design efficient methods for this  memory-distributed setting.

\textbf{Hydra.} In this work we design and analyze the first distributed coordinate descent method: \emph{Hydra: HYbriD cooRdinAte descent.} The method is ``hybrid'' in the sense that it uses parallelism at two levels: i) across a number of nodes in a cluster and ii) utilizing the parallel processing power of individual nodes\footnote{We like to think of each node of the cluster as one of the many heads of the mythological Hydra.}.

Assume we have $c$ nodes (computers) available, each with parallel processing power. In Hydra, we initially partition the coordinates $\{1,2,\dots,d\}$ into $c$  sets, $\B_1, \dots, \B_c$, and assign each set to a single computer. For simplicity, we assume that the partition is balanced: $|\B_k|=|\B_l|$ for all $k,l$. Each computer \emph{owns} the coordinates belonging to its partition for the duration of the iterative process. Also, these coordinates are stored locally. The data matrix describing the problem is partitioned in such a way that all data describing features belonging to $\B_l$ is stored at computer $l$. Now, at each iteration, each computer, independently from the others, chooses a random subset of $\tau$ coordinates from those they own, and computes and applies updates to these coordinates. Hence, once all computers are done, $c\tau$ coordinates will have been updated. The resulting vector, stored as $c$ vectors of size $s=d/c$ each, in a distributed way, is the new iterate. This process is
repeated until convergence. It is important that the computations are done locally on each node, with minimum communication overhead. We comment on this and further details in the text.

\textbf{The main insight.} We show that the parallelization potential of Hydra, that is, its ability to accelerate as $\tau$ is increased, depends on two data-dependent quantities: i) the \emph{spectral norm of the data ($\sigma$)} and ii) a \emph{partition-induced norm of the data ($\sigma'$)}. The first quantity completely describes the behavior of the method in the $c=1$ case. If   $\sigma$ is small, then utilization of more processors (i.e., increasing $\tau$) leads to nearly linear speedup. If $\sigma$ is large, speedup may be negligible, or there may be no speedup whatsoever. Hence, the size of $\sigma$ suggests whether it is worth to use more processors or not. The second quantity, $\sigma'$, characterizes the  effect of the initial partition on the algorithm, and as such is relevant in the $c>1$ case. Partitions with small $\sigma'$ are preferable. For both of these quantities we derive easily computable and interpretable estimates ($\omega$ for $\sigma$ and $\omega'$ for $\sigma'$), which may be used by practitioners to gauge, a-priori, whether their problem of interest is likely to be a good fit for Hydra or not. We show that for strongly convex losses, Hydra outputs an $\epsilon$-accurate solution with probability at least $1-\rho$ after $\tfrac{d\beta}{c \tau \mu} \log (\tfrac{1}{\epsilon \rho})$ iterations (we ignore some small details here), where a single iteration corresponds to changing of $\tau$ coordinates by each of the $c$ nodes; $\beta$ is a stepsize parameter and $\mu$ is a strong convexity constant.

\textbf{Outline.} In Section~\ref{SEC:problem} we describe the structure of the optimization problem we consider in this paper and state assumptions. We then proceed to Section~\ref{sec:algorithm}, in which we describe the method. In Section~\ref{SEC:CONVERGENCE} we prove bounds on the number of iterations sufficient for Hydra to find an  approximate solution with arbitrarily high probability. A discussion of various aspects of our results, as well as a comparison with existing work, can be found in Section~\ref{SEC:LESSONS}. Implementation details of our distributed communication protocol are laid out in Section~\ref{sec:implementation}. Finally, we comment on our computational experiments with a big data (3TB matrix) L1 regularized least-squares instance in Section~\ref{SEC:EXPERIMENTS}.


\section{The problem}\label{SEC:problem}

We study the problem of minimizing regularized loss,
\begin{equation}\label{eq:P}
\min_{x\in \R^d} \Loss(x) \eqdef f(x) + \Reg(x),
\end{equation}
where $f$ is a smooth convex loss, and $\Reg$ is a convex (and possibly nonsmooth) regularizer.

\paragraph{Loss function $f$.} We assume that there exists a positive definite matrix $\bM \in \R^{d \times d}$ such that for all $x,h\in \R^d$,
\begin{equation}\label{eq:kvadraticUpperbound}
 f(x+h) \leq f(x)+ (f'(x))^T h +\tfrac{1}{2} h^T \bM h,
\end{equation}
and write $\bM = \bA^T \bA$, where $\bA$ is some $n$-by-$d$ matrix.

\emph{Example.} These assumptions are natural satisfied in many popular problems. A typical loss function has the form
\begin{equation}\label{eq:SVMproblemFormulation}
f(x) = \textstyle{\sum_{j=1}^n \ell(x, \mrow{\bA}{j}, y^j)},
\end{equation}
where $\bA\in\R^{n \times d}$ is a matrix encoding $n$ examples with $d$ features,
$\mrow{\bA}{j}$ denotes $j$-th row of $\bA$, $\ell$ is some loss function acting on a single example and $y \in \R^n$ is a vector of labels. For instance, in the case of the three losses $\ell$ in Table~1, assumption \eqref{eq:kvadraticUpperbound} holds with $\bM = \bA^T \bA$ for SL and HL, and $\bM = \frac{1}{4} \bA^T \bA$ for LL  \citep{Bradley:PCD-paper}.

\begin{centering}
\begin{table}
\label{tbl:3_losses}
\begin{tabular}{|l|l|}
\hline
square loss (SL)& $  \frac12 (y^j - \mrow{\bA}{j} x )^2$
\\
\hline
logistic loss (LL)&$ \log (1+\exp(- y^j \mrow{\bA}{j} x ) )$
\\
\hline
square hinge  loss (HL)& $ \frac12 \max\{0, 1 - y^j \mrow{\bA}{j} x \}^2$\\
\hline
\end{tabular}
\caption{Examples of loss functions $\ell$ covered by our analysis.}
\end{table}

\end{centering}

\paragraph{Regularizer $\Reg$.} We assume that $\Reg$ is separable, i.e., that it can be decomposed as $ \Reg(x)=\sum_{i=1}^d \Reg_i(x^{i})$, where $x^i$ is the $i$-th coordinate of $x$, and the functions $\Reg_i:\R\to \R \cup \{+\infty\}$ are convex and closed.

\emph{Example.} The choice $\Reg_i(t) = 0$ for $t\in [0,1]$ and $\Reg_i(t)=+\infty$, otherwise, effectively models bound constraints, which are relevant for SVM dual. Other popular choices are $\Reg(x)=\lambda \|x\|_1$ (L1-regularizer) and $\Reg(x)=\tfrac{\lambda}{2} \|x\|_2^2$ (L2-regularizer).

\section{Distributed coordinate descent} \label{sec:algorithm}

We consider a setup with $c$ computers (nods) and first partition the $d$ coordinates (features) into $c$ sets
$\B_1,  \dots, \B_c$ of equal cardinality, $s\eqdef d/c$, and assign set $\B_l$ to node $l$. Hydra is described in Algorithm~\ref{alg:distributedEx}. Hydra's convergence rate depends on the partition; we comment on this later in Sections~\ref{SEC:CONVERGENCE} and \ref{SEC:LESSONS}. Here we simply assume that we work with a fixed partition. We now comment on the steps.

\textbf{Step 3.} At every iteration, each of the $c$ computers picks a random subset of $\tau$ features from those that it owns, uniformly at random, independently of the choice of the other computers. Let $\hat{S}_l$ denote the set picked by node $l$ . More formally, we require that i) $\hat{S}_l \subseteq \B_l$, ii) $\Prob(|\hat{S}_l|=\tau)=1$, where $1\leq \tau\leq s$, and that iii) all subsets of $\B_l$ of cardinality $\tau$ are  chosen equally likely.
In summary, at every iteration of the method, features belonging to the random set $\hat{S}\eqdef \cup_{l=1}^c \hat{S}_l$ are updated. Note that $\hat{S}$ has size $c\tau$, but that, as a sampling from the set $\{1,2,\dots,d\}$, it does not choose all cardinality $c\tau$ subsets of $\{1,2,\dots,d\}$ with equal probability. Hence, the analysis of parallel coordinate descent methods of \citet{RT:PCDM} does not apply. We will say that $\hat{S}$ is a \emph{$\tau$-distributed sampling} with respect to the partition $\{\B_1,\dots,\B_c\}$.

\textbf{Step 4.} Once computer $l$ has chosen its set of $\tau$ coordinates to work on in Step 3, it will \emph{in parallel} compute  (Step 5) and apply (Step 6) updates to them. 

\begin{algorithm}[t!]
\caption{Hydra: HYbriD cooRdinAte descent}
\label{alg:distributedEx}
\SetKwInOut{Initialize}{Parameters}
 \Initialize{$\vt{x}{0} \in \R^d$; $\{\B_1,\dots,\B_c\}$; $\beta>0$, $\tau$; $k \leftarrow 0$\;}

\Repeat{ happy }{
   $x_{k+1} \leftarrow x_k$ \;

  \ParallelForEach{ computer $l \in \{ 1, \ldots, c \} $}{
    Pick a random set of coordinates $\hat{S}_l \subseteq \B_l$ , $|\hat{S}_l|=\tau$ 

    \ParallelForEach{ $i \in \hat{S}_l$}{
      $h_k^i \leftarrow \arg \min_{t} f'_i(x_k)t + \tfrac{M_{ii}\beta}{2}t^2 + \Reg_i(x_k^i + t)$ \;

      Apply the update: $\vt{x}{k+1}^i \leftarrow \vt{x}{k+1}^i +  h^i_k $ \;
    }
    }
  }
\end{algorithm}

\textbf{Step 5.} This is a critical step where updates to coordinates $i \in \hat{S}_l$ are computed. By $f'_i(x)$ we denote the $i$-th partial derivative of $f$ at $x$. Notice that the formula is very simple as it involves one dimensional optimization.

\emph{Closed-form formulas.} Often, $h_k^i$ can be computed in closed form. For $\Reg_i(t) = \lambda_i |t|$ (weighted L1 regularizer),
$h_{k}^i$ is  the point in the interval $[\tfrac{-\lambda_i-f'_i(x_k)}{\bM_{ii}\beta},\tfrac{\lambda_i-f'_i(x_k)}{\bM_{ii}\beta}]$ which is closest to $-x_k^i$.
If $\Reg_i(t) = \tfrac{\lambda_i}{2} t^2$  (weighted L2 regularizer), then $h_k^i = -\tfrac{f'_i(x_k)}{\lambda_i \bM_{ii}\beta}$.

    \emph{Choice of $\beta$.} The choice of the step-size parameter $\beta$ is of paramount significance for the performance of the algorithm, as argued for different but related algorithms by \citet{RT:PCDM, minibatch-ICML2013, FR:SPCDM2013}. We will discuss this issue at length in Sections~\ref{SEC:CONVERGENCE} and \ref{SEC:LESSONS}.

    \emph{Implementation issues:} Note that computer $l$ needs to know the partial derivatives of $f$ at $x_k$ for coordinates $i \in \hat{S}_l \subseteq \B_l$. However, $x_k$, as well as the data describing $f$, is distributed among the $c$ computers. One thus needs to devise a fast and communication efficient way of computing these derivatives. This issue will be dealt with in Section~\ref{sec:implementation}.

\textbf{Step 6.} Here all the $\tau$ updates computed in Step 5 are applied to the iterate. Note that the updates are \emph{local}: computer $l$ only updates coordinates it owns, which are stored locally. Hence, this step is communication-free.

\textbf{Step 7.} Here we are just establishing a way of labeling iterates. That is, starting with $x_k$, all $c$ computers modify $c\tau$ entries of $x_k$ in total, in a distributed way, and the result is called $x_{k+1}$. Our method is therefore inherently \emph{synchronous.} We do not allow, in our analysis, for the various computers to proceed until all computers have updated all coordinates. In practice, a carefully designed asynchronous implementation will be faster, and our experiments in Section~\ref{SEC:EXPERIMENTS} are done with such an implementation.

\section{Convergence rate analysis}\label{SEC:CONVERGENCE}

\emph{Notation:} For any $\bG\in \R^{d\times d}$, let $D^{\bG} = \Diag(\bG)$. That is, $D^\bG_{ii}= \bG_{ii}$ for all $i$ and $D^\bG_{ij}=0$ for $i \neq j$. Further, let $B^{\bG}\in \R^{d\times d}$ be the block diagonal of $\bG $ associated with the partition $\{\B_1,\dots,\B_c\}$. That is, $B^{\bG}_{ij} = \bG_{ij}$ whenever $i,j\in \B_l$ for some $l$, and $B^\bG_{ij} = 0$ otherwise.

\subsection{Four important quantities: $\sigma', \omega', \sigma, \omega$}

Here we define two quantities, $\sigma'$ and $\sigma$, which, as we shall see, play an important role in the computation of the stepsize parameter $\beta$ of Algorithm~\ref{alg:distributedEx}, and through it, in understanding its rate of convergence and potential for speedup by parallelization and distribution. As we shall see, these quantities might not be easily computable. We therefore also provide each with an easily computable and interpretable upper bound, $\omega'$ for $\sigma'$ and $\omega$ for $\sigma$.

Let \begin{equation}\label{eq:Q}\bQ \eqdef (D^\bM)^{-1/2}\bM (D^\bM)^{-1/2},\end{equation}
and notice that, by construction, $\bQ$ has ones on the diagonal. Since $M$ is positive definite, $\bQ$ is as well. For each $l \in \{1,\dots,c\}$, let $\bA_l\in \R^{n\times s}$ be the column submatrix of $\bA$ corresponding to coordinates $i \in \B_l$. The diagonal blocks of $B^\bQ$ are the matrices $\bQ^{ll}$, $l=1,2,\dots,c$, where
\begin{equation}\label{eq:xxx09}\bQ^{kl} \eqdef (D^{\bA_k^T \bA_k})^{-1/2}\bA_k^T \bA_l (D^{\bA_l^T \bA_l})^{-1/2} \in \R^{s \times s}\end{equation}
for each $k,l \in \{1,2,\dots,c\}$.  We now define
\begin{equation}\label{eq:jsdhd8ddj}\sigma' \eqdef \max \{x^T \bQ x \st x \in \R^d, \; x^T B^\bQ x \leq 1\},\end{equation}
\begin{equation}\label{eq:jsdhd8ddj-1}\sigma \eqdef \max \{x^T \bQ x \st x \in \R^d, \; x^T x \leq 1\}.\end{equation}
A useful consequence of \eqref{eq:jsdhd8ddj} is the inequality \begin{equation} \label{eq:sjnsud909323}x^T (\bQ-B^\bQ)x \leq (\sigma'-1) x^T B^\bQ x.\end{equation}

\paragraph{Sparsity.} Let $a_{rl}$ be the $r$-th row of $\bA_l$, and define
\[\omega' \eqdef \max_{1\leq r \leq n} \left\{\omega'(r)\eqdef |\{l  \st l \in \{1,\dots,c\},\; a_{rl} \neq 0\}|\right\},\]
where $\omega'(r)$ is the number of matrices $\bA_l$ with a nonzero in row $r$. Likewise, define
\[\omega \eqdef \max_{1\leq r \leq n} \left\{\omega(r)\eqdef |\{l  \st l \in \{1,\dots,c\},\; \bA_{rl} \neq 0\}|\right\},\]
where $\omega(r)$ is the number of nonzeros in the $r$-th row of $\bA$.

\begin{lemma}\label{lem:shssujs7} The following relations hold:
\begin{equation}\label{eq:sigma-omega}\max\{1,\tfrac{\sigma}{s}\} \leq \sigma' \leq \omega' \leq c, \quad 1\leq \sigma \leq \omega \leq d.\end{equation}
\end{lemma}


\subsection{Choice of the stepsize parameter $\beta$}

We analyze Hydra with  stepsize parameter $\beta\geq \beta^*$, where
\begin{equation}
\begin{split}\beta^* &\eqdef \beta^*_1+\beta^*_2, \\
\beta^*_1 &\eqdef 1+\tfrac{(\tau-1)(\sigma-1)}{s_1},\;\; \beta^*_2 \eqdef \left(\tfrac{\tau}{s} - \tfrac{\tau-1}{s_1}\right) \tfrac{\sigma'-1}{\sigma'}\sigma,
\end{split}\label{eq:beta}
\end{equation}
and $s_1=\max(1,s-1)$. As we shall see in Theorem~\ref{thm:complexity-strongly-convex-case}, fixing $c$ and $\tau$, the number of iterations needed by Hydra find a solution is proportional to $\beta$. Hence, we would wish to use $\beta$ which is as small as possible, but not smaller than the safe choice $\beta=\beta^*$, for which convergence is proved. In practice, $\beta$ can often be chosen smaller than $\beta^*$, leading to larger steps and faster convergence. If the quantities $\sigma$ and $\sigma'$ are hard to compute, then one can replace them by the easily computable upper bounds $\omega$ and $\omega'$, respectively. However, there are cases when $\sigma$ can be efficiently approximated and is much smaller than $\omega$. In some ML datasets with $\bA\in\{0,1\}^{n\times d}$, $\sigma$ is close to the average number of nonzeros in a row of $\bA$, which can be significantly smaller than the maximum, $\omega$. On the other hand, if $\sigma$ is difficult to compute, $\omega$ may provide a good proxy. Similar remarks apply to $\sigma'$. In the $\tau\geq 2$ case (which covers all interesting uses of Hydra), we may ignore $\beta_2^*$ altogether, as implied by the following result.
\begin{lemma}\label{prop:2beta_1} 
If $\tau\geq 2$, then $\beta^* \leq 2\beta_1^*$.
\end{lemma}
This eliminates the need to compute $\sigma'$, at the expense of at most doubling $\beta$, which translates into doubling the number of iterations.

\subsection{Separable approximation}

We first establish a useful identity for the expected value of a random quadratic form obtained by sampling the rows and columns of the underlying matrix via the distributed sampling $\hat{S}$. Note that the result is a direct generalization of Lemma~1 in \citep{minibatch-ICML2013} to the $c>1$ case.

For $x\in \R^d$ and $\emptyset \neq S\subseteq [d] \eqdef \{1,2,\dots,d\}$, we write $x^S \eqdef \sum_{i\in S}  x^{i} e_i$, where $e_i$ is the $i$-th unit coordinate vector. That is, $x^S$ is the vector in $\R^d$ whose coordinates $i\in S$ are identical to those of $x$, but are zero elsewhere.

\begin{lemma}\label{lem:lem8} Fix arbitrary $\bG\in \R^{d \times d}$ and $x \in \R^d$ and let $s_1 = \max(1,s-1)$. Then $\E[(x^{\hat{S}})^T \bG x^{\hat{S}}]$ is equal to
\begin{equation}\label{eq:0966}\tfrac{\tau}{s}\left[ \alpha_1 x^T D^{\bG} x + \alpha_2 x^T \bG x + \alpha_3 x^T (\bG-B^\bG) x
\right],\end{equation}
where $\alpha_1 = 1-\tfrac{\tau-1}{s_1}$, $\alpha_2 = \tfrac{\tau-1}{s_1}$, $\alpha_3 = \tfrac{\tau}{s} - \tfrac{\tau-1}{s_1}$.
\end{lemma}

We now use the above lemma to compute a separable quadratic upper bound on  $\E [ (h^{\hat{S}})^T \bM h^{\hat{S}} ]$.

\begin{lemma}\label{eq:isjs8jss} For all $h \in \R^d$,
\begin{equation}\label{eq:ESO-simple}\E\left[ (h^{\hat{S}})^T \bM h^{\hat{S}} \right] \leq  \tfrac{\tau}{s} \beta^* \left(h^T D^\bM h\right).\end{equation}
\end{lemma}
\begin{proof} For $x\eqdef(D^{\bM})^{1/2}h$, we have $(h^{\hat{S}})^T\bM h^{\hat{S}}=(x^{\hat{S}})^T\bQ x^{\hat{S}}$. Taking expectations on both sides,
and applying Lemma~\ref{lem:lem8}, we see that $\E[(h^{\hat{S}})^T\bM h^{\hat{S}}]$ is equal to \eqref{eq:0966} for $\bG=\bQ$. It remains to bound the three quadratics in \eqref{eq:0966}. Since
$D^{\bQ}$ is the identity matrix, $x^T D^{\bQ}x = h^T D^{\bM} h$. In view of \eqref{eq:jsdhd8ddj-1}, the 2nd term is bounded as $x^T \bQ x \leq \sigma x^T x = \sigma h^T D^{\bM} h$. The last term, $x^T (\bQ-B^\bQ)$, is equal to
\begin{eqnarray}
&=& \tfrac{\sigma'-1}{\sigma'} x^T (\bQ-B^\bQ) x + \tfrac{1}{\sigma'} x^T (\bQ-B^\bQ) x \notag\\
&\overset{\eqref{eq:sjnsud909323}}{\leq}& \tfrac{\sigma'-1}{\sigma'} x^T (\bQ-B^\bQ) x + \tfrac{\sigma'-1}{\sigma'} x^T B^{\bQ} x \notag\\
&=& \tfrac{\sigma'-1}{\sigma'} x^T \bQ x  \overset{\eqref{eq:jsdhd8ddj-1}}{\leq}  \tfrac{\sigma'-1}{\sigma'} \sigma x^T x  =  \tfrac{\sigma'-1}{\sigma'} \sigma h^T D^{\bM} h.\notag
\end{eqnarray}
It only remains to plug in these three bounds into \eqref{eq:0966}.
\end{proof}

Inequalities of type \eqref{eq:ESO-simple} were first proposed and studied by \citet{RT:PCDM}---therein called Expected Separable Overapproximation (ESO)---and were shown to be important for the convergence of parallel coordinate descent methods. However, they studied a different class of loss functions $f$ (convex smooth and partially separable) and different types of random samplings $\hat{S}$, which did not allow them to propose an efficient distributed sampling protocol leading to a distributed algorithm. An ESO inequality was recently used by \citet{minibatch-ICML2013} to design a mini-batch stochastic dual coordinate ascent method (parallelizing the original SDCA methods of \citet{SDCA-2008}) and mini-batch stochastic subgradient descent method (Pegasos of \citet{Pegasos-MAPR}), and give bounds on how mini-batching leads to acceleration. While it was long observed that mini-batching often accelerates Pegasos in practice, it was
only shown with the help of an ESO inequality  that this is so also in theory. Recently, \citet{FR:SPCDM2013} have derived ESO inequalities for smooth approximations of nonsmooth loss functions and hence showed that parallel coordinate descent methods can accelerate on their serial counterparts on a  class of structured nonsmooth convex  losses. As a special case, they obtain a parallel randomized coordinate descent method for minimizing the logarithm of the exponential loss. Again, the class of losses considered in that paper, and the samplings $\hat{S}$, are different from ours. None of the above methods are distributed.

\subsection{Fast rates for distributed learning with Hydra} \label{sec:Hydra-convergence}

Let $x_0$ be the starting point of Algorithm~\ref{alg:distributedEx}, $x_*$ be an optimal solution of problem \eqref{eq:P} and let $\Loss^* = \Loss(x_*)$.
Further, define $\|x\|_{\bM}^2 \eqdef \sum_{i=1}^d \bM_{ii} (x^i)^2$ (a weighted Euclidean norm on $\R^d$) and assume  that $f$ and $\Reg$ are strongly convex
with respect to this norm with convexity parameters $\mu_f$ and $\mu_\Reg$, respectively. A function $\phi$ is strongly convex with parameter $\mu_\phi> 0$ if for all $x,h \in \R^d$,
\[\phi(x+h) \geq \phi(x) + (\phi'(x))^T h + \tfrac{\mu_\phi}{2}\|h\|_{\bM}^2,\]
where $\phi'(x)$ is a subgradient (or gradient) for $\phi$ at $x$.



We now show that  Hydra decreases strongly convex $\Loss$ with an exponential rate in $\epsilon$.

\begin{theorem} \label{thm:complexity-strongly-convex-case}
Assume $\Loss$ is strongly convex with respect to the norm $\|\cdot\|_{\bM}$, with $\mu_f+\mu_\Reg>0$. Choose $x_0\in \R^d$, $0<\rho<1$, $0<\epsilon<\Loss(x_0)-\Loss^*$ and
\begin{equation}\label{eq:k_uniform_strong}
 T \quad \geq \quad \frac{d}{c\tau} \times \frac{\beta+\mu_\Reg}{\mu_f+\mu_\Reg} \times \log \left(\frac{\Loss(x_0)-\Loss^*}{\epsilon\rho}\right),
\end{equation}
where $\beta \geq \beta^*$ and $\beta^*$ is given by \eqref{eq:beta}. If $\{x_k\}$ are the random points generated by Hydra (Algorithm \ref{alg:distributedEx}), then
\[\Prob(\Loss(x_T)-\Loss^*\leq \epsilon) \geq 1-\rho.\]
\end{theorem}
\begin{proof} Outline: We first claim that for all $x,h \in \R^d$,
\begin{equation*}
  \E[f(x+h^{\hat{S}})]
  \leq f(x) + \tfrac{\E[|\hat{S}|]}{d} \left((f'(x))^Th + \tfrac{\beta}{2} h^T D^{\bM} h    \right).
\end{equation*}
To see this, substitute $h\leftarrow h^{\hat{S}}$ into \eqref{eq:kvadraticUpperbound}, take expectations on both sides and then use Lemma~\ref{eq:isjs8jss} together with the fact that for any vector $a$, $\E [a^T h^{\hat{S}}] = \tfrac{\E[|\hat{S}|]}{d} = \tfrac{\tau c}{sc} = \tfrac{\tau}{s}$. The rest follows
by following the steps in the proof in \citep[Theorem~20]{RT:PCDM}.
\end{proof}

A similar result, albeit with the weaker rate $O(\tfrac{s \beta}{\tau \epsilon})$, can be established in the case when neither $f$ nor $\Reg$ are
strongly convex. In big data setting, where parallelism and distribution is unavoidable, it is much more relevant to study the dependence of the rate on parameters such as $\tau$ and $c$. We shall do so in the next section.

\section{Discussion} \label{SEC:LESSONS}

In this section we comment on several aspects of the rate captured in \eqref{eq:k_uniform_strong} and compare Hydra to
selected  methods.

\subsection{Insights into the convergence rate}

Here we comment in detail on the influence of the various design parameters ($c$ = \# computers, $s$ = \# coordinates owned by each computer, and $\tau$ = \# coordinates updated by each computer in each iteration), instance-dependent parameters ($\sigma, \omega, \mu_\Reg, \mu_f$), and parameters depending both on the instance and design ($\sigma', \omega'$),  on the stepsize parameter $\beta$, and through it, on the convergence rate described in Theorem~\ref{thm:complexity-strongly-convex-case}.

\paragraph{Strong convexity.} Notice that the size of $\mu_\Reg>0$ mitigates the effect of a possibly large $\beta$ on the bound \eqref{eq:k_uniform_strong}. Indeed, for large $\mu_\Reg$, the factor $(\beta+\mu_\Reg)/(\mu_f + \mu_\Reg)$ approaches 1, and the bound \eqref{eq:k_uniform_strong} is dominated by the term $\tfrac{d}{c\tau}$, which means that Hydra enjoys linear speedup in $c$ and $\tau$. In the following comments we will assume that $\mu_\Reg=0$, and focus on studying the dependence of the leading term $d\tfrac{\beta}{c\tau}$ on various quantities, including $\tau, c, \sigma$ and $\sigma'$.

\paragraph{Search for small but safe $\beta$.} As shown by \citet[Section~4.1]{minibatch-ICML2013},  mini-batch SDCA might \emph{diverge} in the setting with $\mu_f=0$ and $\Reg(x)\equiv 0$, even for a simple quadratic function with $d=2$, provided that $\beta=1$. Hence, small values of $\beta$ need to be avoided. However, in view of Theorem~\ref{thm:complexity-strongly-convex-case}, it is good if $\beta$ is as small as possible. So, there is a need for a ``safe'' formula for a small $\beta$. Our formula \eqref{eq:beta}, $\beta=\beta^*$, is serving that purpose. For a detailed  introduction into the issues related to selecting a good $\beta$ for parallel coordinate descent methods, we refer the reader to the first 5 pages of \citep{FR:SPCDM2013}.

\begin{table}
\begin{centering}
{\footnotesize
\begin{tabular}{|c|c|c|}
  \hline
   special case & $\beta^*$ & $\beta^* / (c\tau)$\\
  \hline\begin{tabular}{c}
          any $c$   \\
          $\tau=1$ \\
        \end{tabular}
     & $1 + \tfrac{\sigma}{s}\left(\tfrac{\sigma'-1}{\sigma'}\right)$ & $ s + \sigma\left(\tfrac{\sigma'-1}{\sigma'}\right)$\\
  \hline
  \begin{tabular}{c}
          $c=1$  \\
        any $\tau$ \\
        \end{tabular}
     & $1+\tfrac{(\tau-1)(\sigma-1)}{d-1}$ & $\tfrac{d}{\tau}\left(1+ \tfrac{(\tau-1)(\sigma-1)}{d-1}\right)$\\
       \hline
    $\tau c=d$         & $\sigma$  & $\sigma$ \\
  \hline
\end{tabular}
\caption{Stepsize parameter $\beta=\beta^*$ and the leading factor in the rate \eqref{eq:k_uniform_strong} (assuming $\mu_\Reg=0$) for several special cases of Hydra.}
\label{tbl:beta}
}
\end{centering}
\end{table}


\paragraph{The effect of $\sigma'$.}  If $c = 1$, then by Lemma~\ref{eq:sigma-omega}, $\sigma'=c=1$, and hence $\beta^*_2=0$. However, for $c>1$ we \emph{may} have $\beta^*_2>0$, which can hence be seen as a price we need to pay for using more nodes. The price depends on the way the data is partitioned to the nodes, as captured by $\sigma'$. In favorable circumstances, $\sigma'\approx 1$ even if $c>1$, leading to $\beta_2^* \approx 0$. However, in general we have the bound $\sigma'\geq \tfrac{c\sigma}{d}$, which gets worse as $c$ increases and, in fact, $\sigma'$ can be as large as $c$. Note also that $\xi$ is decreasing in $\tau$, and that $\xi(s,s)=0$. This means that by choosing $\tau=s$ (which effectively removes randomization from Hydra), the effect of $\beta^*_2$ is eliminated. This may not be always possible as often one needs to solve problems with $s$ vastly larger than the number of updates that can be performed on any given node in parallel. If $\tau\ll s$, the effect of $\beta^*_2$ can be controlled, to a certain extent, by choosing a partition with small $\sigma'$. Due to the way $\sigma'$ is defined, this may not be an easy task. However, it may be easier to find partitions that minimize $\omega'$, which is often a good proxy for $\sigma'$. Alternatively, we may ignore estimating $\sigma'$ altogether by setting $\beta=2\beta_1^*$, as mentioned before, at the price of at most doubling the number of iterations.

\paragraph{Speedup by increasing $\tau$.} Let us fix $c$ and compare the quantities $\gamma_\tau\eqdef \tfrac{\beta^*}{c\tau}$ for $\tau=1$ and $\tau=s$. We now show that $\gamma_1 \geq \gamma_s$, which means that if all coordinates are updated at every node, as opposed to one only, then Hydra run with $\beta=\beta^*$ will take fewer iterations. Comparing the 1st and 3rd row of Table~\ref{tbl:beta}, we see that $\gamma_1 = s+ \sigma \tfrac{\sigma'-1}{\sigma'}$ and $\gamma_s = \sigma$. By Lemma~\ref{lem:shssujs7}, $\gamma_1-\gamma_s = s - \tfrac{\sigma}{\sigma'}\geq 0$.


{\footnotesize
\begin{figure}[htp]
 \centering
 \subfigure[$(c,s) = (1,10^5)$]{ \includegraphics[width=1.5in]{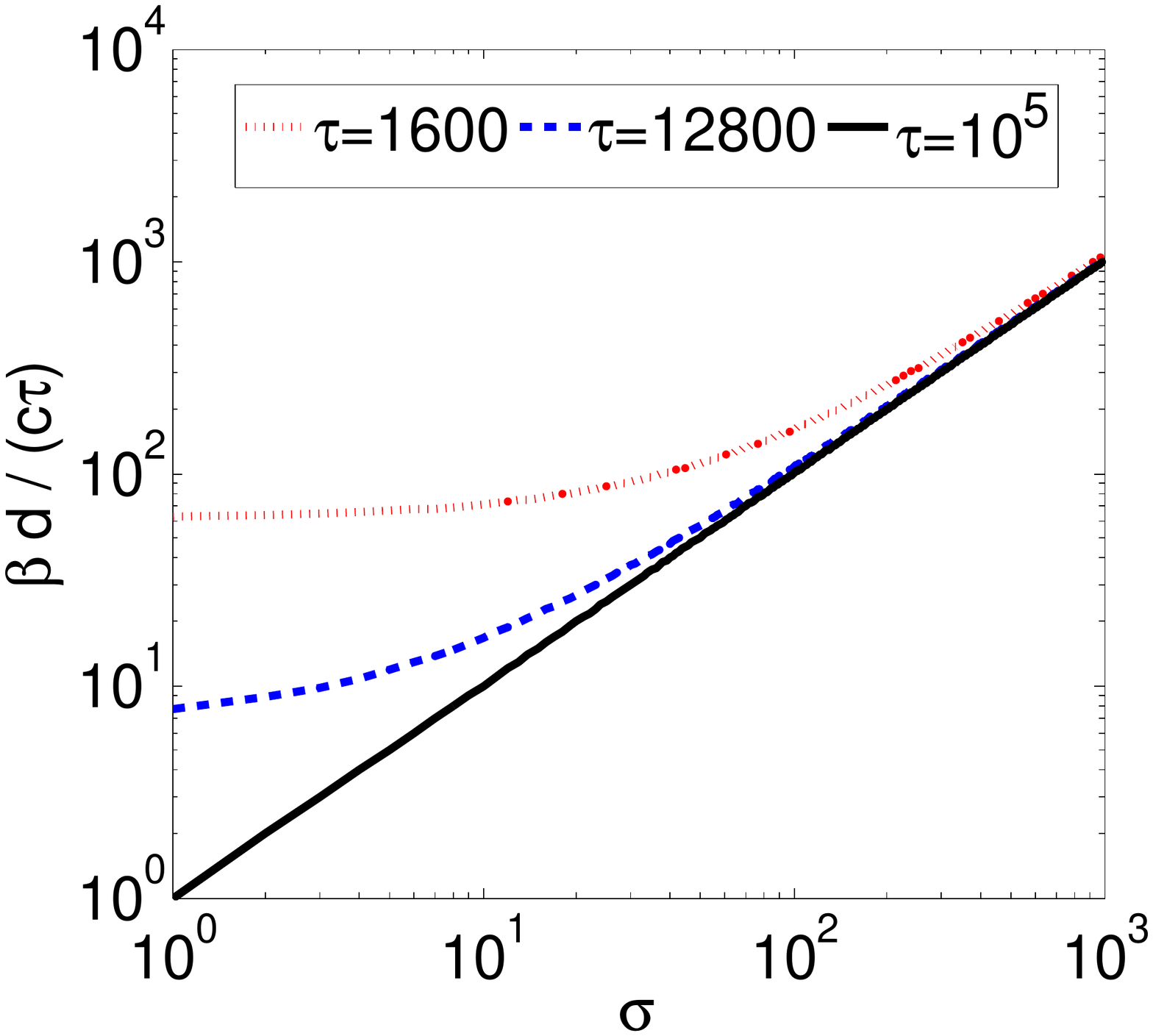}}
 \subfigure[$(c,s)=(10^2,10^3)$]{\includegraphics[width=1.5in]{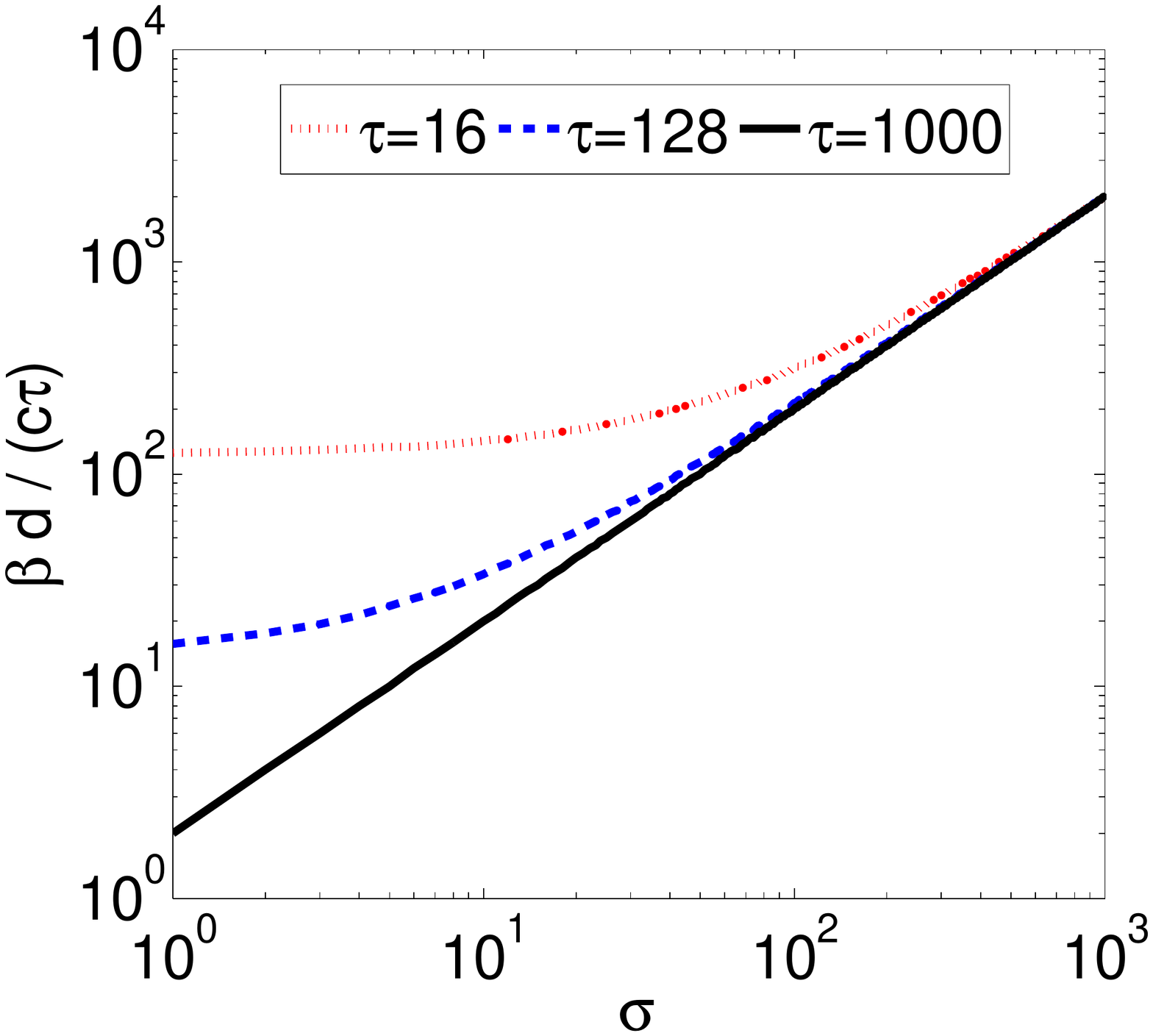}}
\caption{In terms of the number of iterations, very little is lost by using $c>1$ as opposed to $c=1$.}
 \label{fig:reduceAll}
\end{figure}
}


\paragraph{Price of distribution.} For illustration purposes, consider a problem with $d=10^5$ coordinates. In Figure~\ref{fig:reduceAll}(a) we depict the size of $\tfrac{d\beta_1^*}{c\tau}$ for $c=1$ and several choices of $\tau$, as a function of $\sigma$. We see that Hydra works better for small values of $\sigma$ and that with increasing $\sigma$, the benefit of using updating more coordinates diminishes.  In Figure~\ref{fig:reduceAll}(a) we consider the same scenario, but with $c=100$ and $s=1000$, and we plot $\tfrac{d2\beta_1^*}{c\tau}$ on the $y$ axis. Note that the red dotted line in both plots corresponds to a parallel update of 1600 coordinates. In (a) all are updated on a single node, whereas in (b) we have 100 nodes, each updating 16 coordinates at a time. Likewise, the dashed blue dashed and solid black lines are also comparable in both plots. Note that the setup with $c=10$ has a slightly weaker performance, the lines are a bit lower. This is the price we pay for using $c$ nodes as opposed to a single node (obviously, we are ignoring communication cost here). However, in big data situations one simply has no other choice but to utilize more nodes.


\subsection{Comparison with other methods}

While we are not aware of any other \emph{distributed} coordinate descent method, Hydra in the $c=1$ case is closely related to several existing parallel coordinate descent methods.

\paragraph{Hydra vs Shotgun.} The Shotgun algorithm (parallel coordinate descent) of \citet{Bradley:PCD-paper} is similar to Hydra for $c=1$. Some of the differences: \citet{Bradley:PCD-paper} only consider $\Reg$ equal to the $L1$ norm and their method works in dimension $2d$ instead of the native dimension $d$. Shotgun was not analyzed for strongly convex $f$, and convergence in expectation was established. Moreover, \citet{Bradley:PCD-paper} analyze the step-size choice $\beta = 1$, fixed independently of the number of parallel updates $\tau$, and give results that hold only in a ``small $\tau$'' regime. In contrast, our analysis works for any choice of $\tau$.

\paragraph{Hydra vs PCDM.} For $c=1$, Hydra reduces to the parallel coordinate descent method (PCDM) of \citet{RT:PCDM}, but with a \emph{better} stepsize parameter $\beta$. We were able to achieve smaller $\beta$ (and hence better rates) because we analyze a different and more specialized class of loss functions (those satisfying \eqref{eq:kvadraticUpperbound}). In comparison, \citet{RT:PCDM} look at a general class of partially separable losses. Indeed, in the $c=1$ case, our distributed sampling $\hat{S}$ reduces to the sampling considered in \citep{RT:PCDM} ($\tau$-nice sampling). Moreover, our formula for $\beta$ (see Table~\ref{tbl:beta}) is essentially identical to the formula for $\beta$ provided in \citep[Theorem~14]{RT:PCDM}, with the exception that we have $\sigma$ where they have $\omega$. By \ref{eq:sigma-omega}, we have $\sigma\leq \omega$, and hence our $\beta$ is smaller.

\paragraph{Hydra vs SPCDM.} SPCDM of \cite{FR:SPCDM2013} is PCDM applied to a smooth approximation of a nonsmooth convex loss; with a special choice of $\beta$, similar to $\beta_1$. As such, it extends the reach of PCDM to a large class of nonsmooth losses, obtaining $O(\tfrac{1}{\epsilon^2})$ rates.

\paragraph{Hydra vs mini-batch SDCA.}  \citet{minibatch-ICML2013} studied the performance of a mini-batch stochastic dual coordinate ascent for SVM dual (``mini-batch SDCA''). This is a special case of our setup with $c=1$, convex quadratic $f$ and $\Reg_i(t) = 0$ for $t \in [0,1]$ and $\Reg_i(t)=+\infty$ otherwise. Our results can thus be seen as a generalization of the results in that paper to a larger class of loss functions $f$, more general regularizers $\Reg$, and most importantly, to the distributed setting ($c>1$). Also, we give $O(\log \tfrac{1}{\epsilon})$ bounds under strong convexity, whereas \citep{minibatch-ICML2013} give $O(\tfrac{1}{\epsilon})$ results without assuming strong convexity. However,  \citet{minibatch-ICML2013} perform a primal-dual analysis, whereas we do not.

\section{Distributed computation of the gradient}\label{sec:implementation}

In this section we described some important elements of our distributed implementation.


\begin{table}[!tp]
 \centering
 \small
 \begin{tabular}{|c|c|c|}
 \hline
  $\ell$       & $f'_i(x)$ & $\bM_{ii}$ \\
         \hline
 SL      & $ \sum_{j=1}^m \quad - \bA_{ji} (y^j - \mrow{\bA}{j} x)$ 					  & $ \|\mcol{\bA}{i}\|_2^2$		\\
 LL      & $ \sum_{j=1}^m \quad -  y^j \bA_{ji} \frac{ \exp(-y^j \mrow{\bA}{j} x)} {1 + \exp( - y^j \mrow{\bA}{j} x)}$ & $ \frac14  \|\mcol{\bA}{i}\|_2^2$ \\
 HL  & $ \sum_{j \st  y^j \mrow{\bA}{j} x < 1}   \left(  - y^j \bA_{ji} (1 - y^j \mrow{\bA}{j} x)\right)$  		  & $ \|\mcol{\bA}{i}\|_2^2$\\
 \hline
\end{tabular}
\caption{Information needed in Step 5 of Hydra for $f$ given by \eqref{eq:SVMproblemFormulation} in the case of the three losses $\ell$ from Table~1.}\label{tbl:lossFunctions2}
\end{table}

Note that in Hydra, $x_k$ is stored in a distributed way. That is, the  values  $x_k^i$ for $i \in \B_l$ are stored on computer $l$. Moreover, Hydra partitions $\bA$ columnwise as $\bA=[\bA_1,\dots,\bA_c]$, where $\bA_l$ consists of columns $i \in \B_l$ of $\bA$, and stores $\bA_l$ on computer $l$. So,  $\bA$ is chopped into smaller pieces with stored in a distributed way in fast memory (if possible) across the $c$ nodes. Note that this allows the method to work with large matrices.

At Step 5 of Hydra, node $l$ at iteration $k+1$ needs to know the partial derivatives $f'_i(x_{k+1})$ for $i \in \hat{S}_l \subseteq \B_l$. We now describe several efficient distributed protocols for the computation of $f'_i(x_{k+1})$ for functions $f$ of the form \eqref{eq:SVMproblemFormulation}, in the case of the three losses $\ell$ given in Table~1 (SL, LL, HL). The formulas for $f'_i(x)$ are summarized in Table~3 ($\bA_{j:}$ refers to the $j$-th row of $\bA$). Let $D^y \eqdef \Diag(y)$.

\subsection{Basic protocol} If we write $h_k^i = 0$ if $i$ is not updated in iteration $k$, then
\begin{equation}\label{eq:hhsjhsjs}
\vt{x}{k+1} = \vt{x}{k} + \sum_{l=1}^c \sum_{i \in \hat{S}_l}  h^i_k e_i.
\end{equation}
Now, if we let
\begin{align}
\label{eqn:residuals}
\vt{g}{k} \eqdef & \; \begin{cases}
 \; \bA \vt{x}{k} - y, & \mbox{ for SL},\\
 \; - D^y \bA \vt{x}{k},  & \mbox{ for LL and HL},
   \end{cases}
\end{align}
then by combining \eqref{eq:hhsjhsjs} and \eqref{eqn:residuals}, we get
\[g_{k+1}= g_k +\sum_{l=1}^c \delta g_{k,l}, \qquad \text{where}\]
\[\delta g_{k,l} = \begin{cases} \sum_{i \in \hat{S}_l} h_k^i \mcol{\bA}{i}, & \text{for SL}, \\
\sum_{i \in \hat{S}_l}       -h_k^i D^y \mcol{\bA}{i}, & \text{for LL and HL}.
\end{cases}\]
Note that the value $\delta g_{k,l}$ can be computed on node $l$ as all the required data is stored locally. Hence, we let each node compute $\delta g_{k,l}$, and then use a {\it reduce all} operation to add up the updates to obtain $g_{k+1}$, and pass the sum  to all nodes. Knowing $g_{k+1}$, node $l$ is then able to compute $f'_i(x_{k+1})$ for any $i \in \B_l$  as follows:
\begin{align*}
f'_i(x_{k+1}) = & \;
\begin{cases}
 \; \mcol{\bA}{i}^T g_{k+1} = \sum_{j=1}^n \; \bA_{ji} g_{k+1}^j, & \mbox{ for SL},\\   
 \; \sum_{j=1}^n \; y^j \bA_{ji}  \frac{ \exp( g_{k+1}^j) }{1+\exp( g_{k+1}^j )},  &\mbox{ for LL},\\
 \; \sum_{j \st g_{k+1}^j > -1} \; y^j \bA_{ji} (1 + g_{k+1}^j ), &\mbox{ for HL}.
   \end{cases}
\end{align*}

\subsection{Advanced protocols}

The basic protocol discussed above has obvious drawbacks. Here we identify them and propose modifications leading to better performance.
\begin{itemize}
 \item {\it alternating Parallel and Serial regions (PS):} The basic protocol alternates between two procedures: i) a computationally heavy one (done in parallel) with no MPI communication, and ii)  MPI communication (serial). An easy fix would be to dedicate 1 thread to deal with communication and the remaining threads within the same computer for computation. We call this protocol \emph{ Fully Parallel (FP)}. Figure~\ref{fig:serialBlocksAndParallel} compares the basic (left) and FP (right) approaches.

     \useFIG{
\begin{figure}[htp]
 \centering
 \includegraphics[width=3.1in]{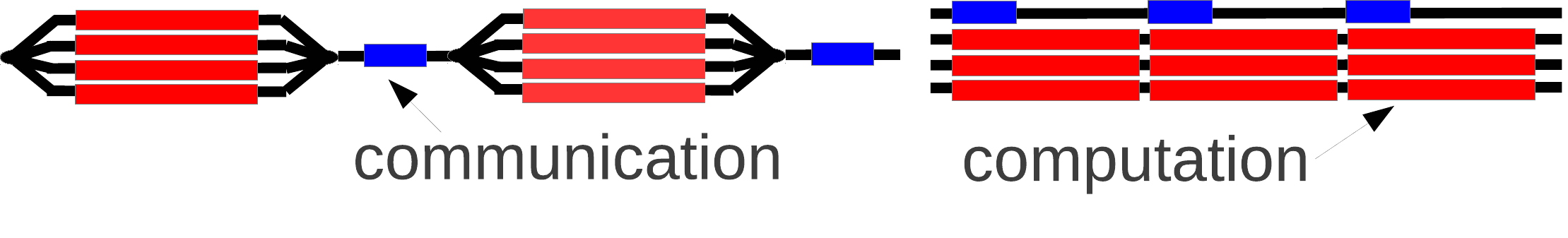}
 \caption{Parallel-serial (PS; left) vs Fully Parallel (FP; right) approach.}
 \label{fig:serialBlocksAndParallel}
\end{figure}
}

 \item {\it Reduce All (RA):} In general, reduce all operations may significantly degrade the performance of distributed algorithms. Communication taking place only between nodes close to each other in the network, e.g., nodes directly connected by a cable, is more efficient. Here we propose the \emph{Asynchronous StreamLined (ASL)} communication protocol in which each node, in a given iteration, sends only 1 message (asynchronously) to a nearby computer, and also receives only one message (asynchronously) from another nearby computer. Communication hence takes place in an {\em Asynchronous Ring}. This communication protocol requires significant changes in the algorithm. Figure~\ref{fig:ASL} illustrates the flow of messages at the end of the $k$-th iteration for $c=4$.

\useFIG{
 \begin{figure}[htp]
 \centering
 \includegraphics[width=2in]{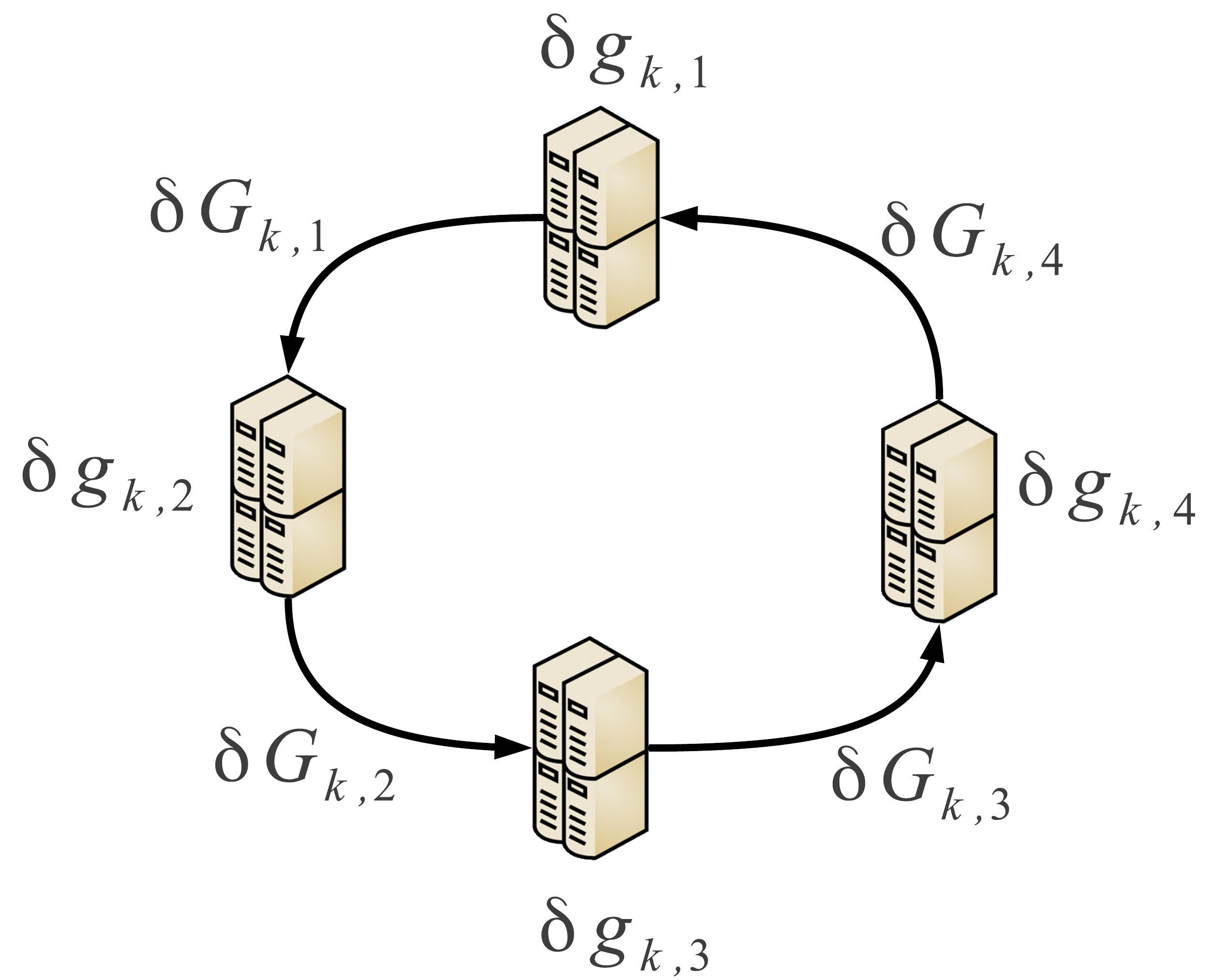}
 \caption{ASL protocol with $c=4$ nodes. In iteration $k$, node $l$ computes $\delta g_{k,l}$,
and sends $\delta G_{k,l}$ to $l_+$.}
 \label{fig:ASL}
\end{figure}
}

We order the nodes into a ring, denoting $l_-$ and $l_+$ the two nodes neighboring node $l$. Node $l$ only receives data  from $l_-$, and sends data to  $l_+$. Let us denote by $\delta G_{k,l}$ the data sent by node $l$  to  $l_+$ at the end of iteration $k$. When  $l$ starts iteration $k$, it already knows $\delta G_{k-1,l_-}$.\footnote{Initially, we let $\delta g_{k,l} = \delta G_{k,l} = 0$ for all $k \leq 0$.} Hence, data which will be sent at the end of the $k$-th iteration by node $l$ is given by
\begin{equation}\label{eq:cumulativeupdate}
\delta G_{k,l} = \delta G_{k-1,l_-} - \delta g_{k-c, l} + \delta g_{k,l}.
\end{equation}
This leads to the update rule
\[
  g_{k+1,l} = g_{k,l} + \delta g_{k,l}  + \delta G_{k,l_-} - \delta g_{k-c+1, l}.
\]

ASL needs less communication per iteration. On the other hand, information is propagated more slowly to the nodes through the ring, which may adversely affect the number of iterations till convergence (note that we do not analyze Hydra with this communication protocol). Indeed, it takes $c-1$
iterations to propagate information to all nodes. Also, storage requirements have increased: at iteration $k$ we need to store the  vectors
$\delta g_{t,l}$ for $k-c \leq t \leq k$ on computer $l$.

\end{itemize}

\section{Experiments} \label{SEC:EXPERIMENTS}

In this  section we present numerical evidence that Hydra is capable to efficiently solve big data problems. We have a C++ implementation, using Boost::MPI and OpenMP. Experiments were executed on a Cray XE6 cluster with 128 nodes; with each node equipped with two AMD Opteron Interlagos 16-core processors and 32 GB of RAM. We consider a LASSO problem, i.e., $f$ given by \eqref{eq:SVMproblemFormulation} with $\ell$ being the square loss (SL) and $\Reg(x) = \|x\|_1$. In order to to test Hydra under controlled conditions, we adapted the LASSO generator proposed by \citet[Section 6]{Nesterov:2007composite}; modifications were necessary as the generator does not work well in the big data setting.

\begin{table}[htp]
 \centering
 \small
 \footnotesize
\begin{center}
\begin{tabular}{c|c|c||r|r}
$\TT$ & comm. protocol & organization & avg. time & speedup
 \\ \hline \hline

$10$&RA & PS & 0.040 & ---
\\
$10$&RA & FP & 0.035 & 1.15
\\
$10$&ASL& FP & 0.025 & 1.62
\\ \hline
$10^2$&RA & PS & 0.100 & ---
\\
$10^2$&RA & FP & 0.077 & 1.30
\\
$10^2$&ASL& FP & 0.032 & 3.11
\\ \hline
$10^3$&RA & PS & 0.321 & ---
\\
$10^3$&RA & FP & 0.263 & 1.22
\\
$10^3$&ASL& FP & 0.249 & 1.29\\
\hline
\end{tabular}
\end{center}
 \caption{Duration of a single Hydra iteration for 3 communication protocols.
The basic RA-PS protocol is always the slowest, but follows the theoretical analysis. ASL-FP can be 3$\times$ faster.}
 \label{tbl:tweaks}
\end{table}

\paragraph{Basic communication protocol vs advanced protocols.}
As discussed in Section~\ref{sec:implementation}, the advantage of the RA protocol is the fact that Theorem~\ref{thm:complexity-strongly-convex-case} was proved  in this setting, and hence can be used as a safe benchmark for comparison with the advanced protocols.


Table~\ref{tbl:tweaks} compares the average time per iteration for the 3 approaches and 3 choices of $\tau$. We used $128$ nodes, each running 4 MPI processes (hence $c=512$). Each MPI process runs 8 OpenMP threads, giving 4,096 cores in total. The data matrix $\bA$ has $n=10^9$ rows and $d = 5 \times 10^8$ columns, and  has 3 TB,  double precision. One can observe that in all cases, ASL-FP yields largest gains compared to the benchmark RA-PS protocol. Note that ASL has some overhead in each iteration, and hence  in cases when computation per node is small ($\tau=10$), the speedup is only 1.62. When $\tau=10^2$ (in this case the durations of computation and communication were comparable),  ASL-FP is 3.11 times faster than RA-PS. But the gain becomes again only moderate for $\tau=10^3$; this is because computation now takes much longer than communication, and hence the choice of strategy for updating the auxiliary vector $g_k$ is less significant. Let us  remark that the use of larger $\tau$ requires larger $\beta$, and hence possibly more iterations (in the worst case).

\paragraph{Huge LASSO problem.} We generated a sparse matrix $\bA$ with block angular structure, depicted in \eqref{eq:block_ang}.

\begin{equation}\label{eq:block_ang}
\bA=\left(
\begin{BMAT}(rc){c;c;c;c}{ccc;c}
\bA_{1}^{loc} & 0   & \cdots & 0\\
0 & \bA_{2}^{loc}   & \cdots & 0\\
\vdots & \vdots    & \ddots & \vdots\\
\bA^{glob}_{1} & \bA^{glob}_{2}   & \cdots & \bA^{glob}_{c}
\end{BMAT}
\right).
\end{equation}

Such matrices often arise in stochastic optimization. Each Hydra head (=node) $l$ owns two matrices: $\bA^{loc}_{l} \in \R^{1,952,148 \times 976,562}$
and $\bA^{glob}_{l} \in \R^{500,224 \times 976,562}$. The average number of nonzero elements per row in the local part of $\bA_l$ is $175$, and $1,000$ for the global part. Optimal solution $x_*$ has exactly $160,000$ nonzero elements. Figure~\ref{fig:ASL_vs_RA} compares the evolution of $\Loss(x_k)-\Loss^*$ for  ASL-FP and RA-FP.

\emph{Remark:} When communicating $g_{kl}$, only entries corresponding to the global part of $\bA_l$  need to be communicated, and hence in RA, a {\it reduce all} operation is applied to vectors $\delta g_{glob,l} \in \R^{500,224}$. In ASL, vectors with the same length are sent.

{\footnotesize
\begin{figure}[htp]
 \centering
 \includegraphics[width=3.3in]{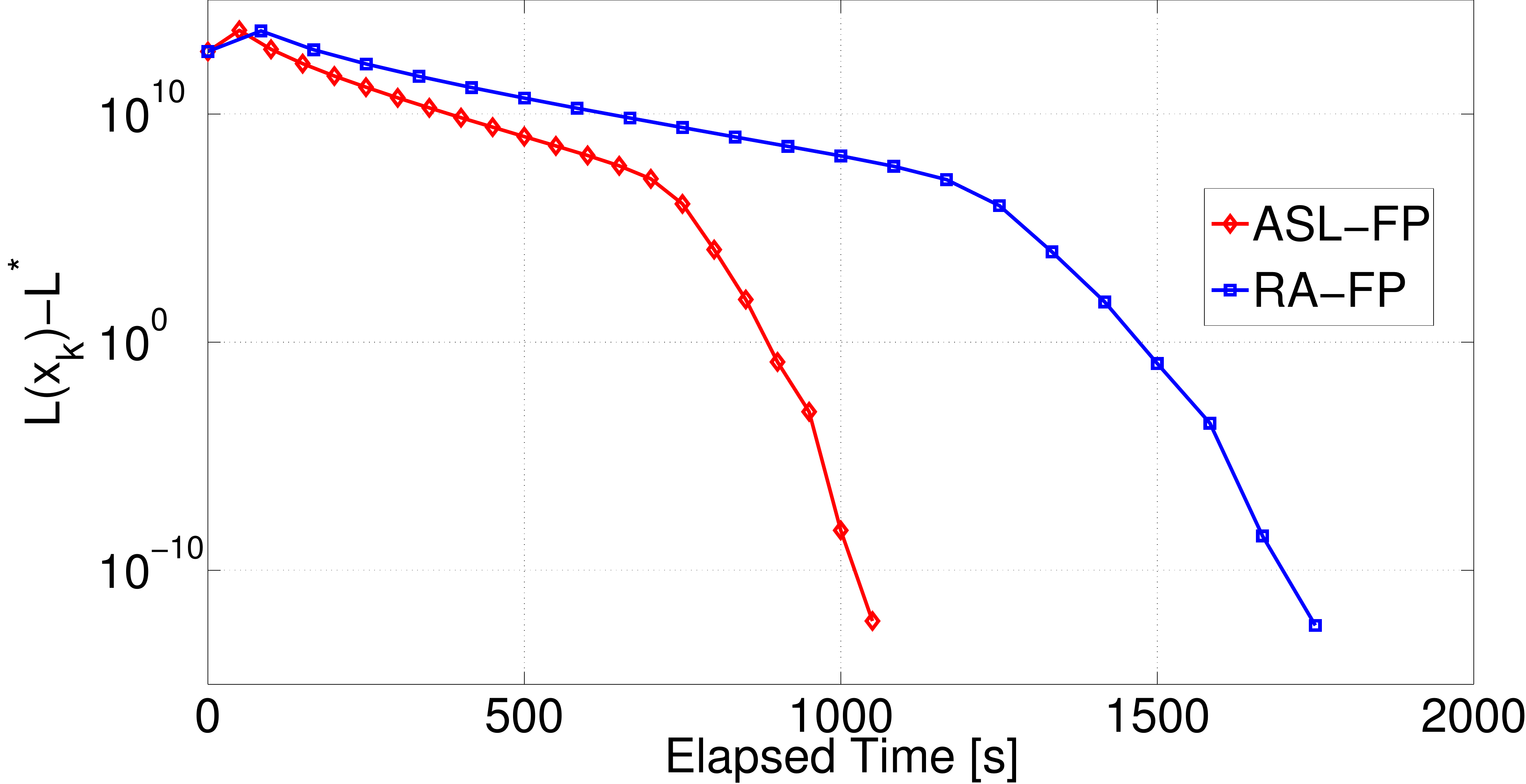}
 \caption{Evolution of $\Loss(x_k)-\Loss^*$ in time. ASL-FP significantly  outperforms  RA-FP. The loss $\Loss$ is pushed down by  25 degrees of magnitude in less than 30 minutes (3TB problem).}
 \label{fig:ASL_vs_RA}
\end{figure}
}

\section{Extensions}

Our results can be extended to the setting where coordinates are replaced by blocks of coordinates, as in \citep{Nesterov:2010RCDM}, and to partially separable losses, as in \citep{RT:PCDM}.

\bibliographystyle{icml2014}
\bibliography{paper}


\clearpage
\onecolumn
\appendix

\section{Proof Lemma~\ref{lem:shssujs7}}

\begin{enumerate}
\item The inequality $\omega' \leq c$ is obviously true. By considering $x$ with zeroes in all coordinates except those that belong to $\B_k$ (where $k$ is an arbitrary but fixed index),  we see that $x^T \bQ x = x^T B^\bQ x$, and hence $\sigma'\geq 1$.

\item We now establish that $\sigma'\leq \omega'$.  Let $\phi(x) = \tfrac{1}{2}x^T \bQ x$, $x\in \R^d$; its gradient is
\begin{equation}\label{eq:ssjdd9}\phi'(x) = \bQ x.\end{equation}
 For each $k=1,2,\dots,c$, define a pair of conjugate norms on $\R^s$ as follows:
\begin{equation}\label{eq:sjs8sns8s}\|v\|_{(k)}^2 \eqdef \ve{\bQ^{kk}v}{v}, \qquad (\|v\|_{(k)}^*)^2 \eqdef \max_{\|v'\|_{(k)}\leq 1} \ve{v'}{v} = \ve{(\bQ^{kk})^{-1}v}{v}. \end{equation}

Let $\bU_k$ be a column submatrix of the $d$-by-$d$ identity matrix corresponding to columns $i \in \B_k$. Clearly, $\bA_k = \bA \bU_k$ and $\bU_k^T \bQ \U_k$ is the $k$-th diagonal block of $\bQ$, i.e.,
 \begin{equation}\label{eq:jhd009909}\bU_k^T \bQ \bU_k \overset{\eqref{eq:Q}}{=}  \bQ^{kk}.\end{equation}
Moreover, for $x \in \R^d$ and $k \in \{1,2,\dots,c\}$, let $x^{(k)} = \bU_k^T x$ and, fixing positive scalars $w_1,\dots,w_c$, define a norm on $\R^d$ as follows:
\begin{equation}\label{eq:block_norm}\|x\|_w \eqdef \left(\sum_{k=1}^c w_k \|x^{(k)}\|_{(k)}^2 \right)^{1/2}.\end{equation}

Now, we claim that for each $k$,
\[\|\bU_k^T \phi'(x+\bU_k h^{(k)}) - \bU_k^T \phi'(x)\|_{(k)}^* \leq  \|h^{(k)}\|_{(i)}.\]
This means that $\phi'$ is block Lipschitz (with blocks corresponding to variables in $\B_k$), with respect to the norm $\|\cdot\|_{(k)}$, with Lipschitz constant $1$. Indeed, this is, in fact, satisfied with equality:
\begin{eqnarray*}
 \|\bU_k^T \phi'(x+\bU_k h^{(k)}) - \bU_k^T \phi'(x)\|_{(k)}^* &\overset{\eqref{eq:ssjdd9}}{=}& \|\bU_k^T \bQ (x+\bU_k h^{(k)}) - \bU_k \bQ x\|_{(k)}^*\\
&=& \|\bU_k^T  \bQ \bU_k h^{(k)}\|_{(k)}^* \\
&\overset{\eqref{eq:jhd009909}}{=}&   \|\bQ^{kk} h^{(k)}\|_{(k)}^* \\
& \overset{\eqref{eq:sjs8sns8s}}{=} & \ve{(\bQ^{kk})^{-1} \bQ^{kk} h^{(k)}}{ \bQ^{kk} h^{(k)}} \quad \overset{\eqref{eq:sjs8sns8s}}{=} \quad \|h^{(k)}\|_{(k)}.\end{eqnarray*}

This is relevant because then, by  \citet[Theorem 7; see comment 2 following the theorem]{RT:PCDM}, it follows that $\phi'$ is Lipschitz with respect to $\|\cdot\|_{w}$, where $w_k=1$ for all $k=1,\dots,c$, with Lipschitz constant $\omega'$ ($\omega'$ is the degree of partial block separability of $\phi$ with respect to the blocks $\B_k$). Hence,

\[\tfrac{1}{2}x^T \bQ x = \phi(x) \leq \phi(0) + (\phi'(0))^T x + \frac{\omega'}{2}\|x\|^2_{w} \overset{\eqref{eq:sjs8sns8s}+\eqref{eq:block_norm}}{=} \frac{\omega'}{2}\sum_{k=1}^c \ve{\bQ^{kk}x^{(k)}}{x^{(k)}} = \frac{\omega'}{2} (x^T B^{\bQ} x) ,\]
which establishes the inequality $\sigma'\leq \omega'$.

\item We now show that $\tfrac{\sigma}{s}\leq \sigma'$. If we let $\theta \eqdef \max\{x^T B^\bQ x : x^T x \leq 1\}$, then $x^T B^\bQ x \leq \theta
    x^T x$ and hence $\{x \st x^T x \leq 1\} \subseteq \{x \st x^T B^\bQ x \leq \theta\}$. This implies that
    \[\sigma = \max_x \{x^T \bQ x \st x^Tx \leq 1\} \leq \max_x \{x^T \bQ x \st x^T B^\bQ x \leq \theta\} = \theta \sigma'.\]
    It now only remains to argue that $\theta \leq s$. For $x\in \R^d$, let $x^{(k)}$ denote its subvector in $\R^s$ corresponding to coordinates $i \in \B_k$ and $\Delta = \{p \in \R^c : p \geq 0, \; \sum_{k=1}^c p_k =1\}$. We can now write
    \begin{eqnarray*}
    \theta &=&  \max_x \left\{ \sum_{k=1}^c (x^{(k)})^T \bQ^{kk} x^{(k)} \st \sum_{k=1}^c (x^{(k)})^T x^{(k)} \leq 1\right\}\\
    &=& \max_{p\in \Delta}  \sum_{k=1}^c \left\{ \max (x^{(k)})^T \bQ^{kk} x^{(k)} \st (x^{(k)})^T x^{(k)} = p_k \right\}\\
    &=& \max_{p\in \Delta}  \sum_{k=1}^c  p_k \max \left\{(x^{(k)})^T \bQ^{kk} x^{(k)} \st (x^{(k)})^T x^{(k)} = 1 \right\}\\
    &=& \max_{1\leq k \leq c} \max \left\{(x^{(k)})^T \bQ^{kk} x^{(k)} \st (x^{(k)})^T x^{(k)} = 1 \right\}    \quad \leq \quad s.
    \end{eqnarray*}
   In the last step we have used the fact that $\sigma(\bQ)  =\sigma \leq c = \dim(\bQ)$, proved in steps 1 and 2, applied to the setting $\bQ\leftarrow \bQ^{kk}$.

\item The chain of inequalities $1\leq \sigma \leq \omega \leq c$ is obtained as a special case of the chain $1\leq \sigma' \leq \omega' \leq d$ (proved above) when
$c=d$ (and hence $\B_l = \{l\}$ for $l=1,\dots,d$). Indeed, in this case $B^\bQ = D^\bQ$, and so $x^T B^\bQ x = x^T D^\bQ x = x^T x$, which means that
$\sigma'=\sigma$ and $\omega'=\omega$.
\end{enumerate}

\section{Proof of Lemma~\ref{prop:2beta_1}}

It is enough to argue that $\beta_2^* \leq \beta_1^*$. Notice that $\beta_2^*$ is increasing in $\sigma'$. On the other hand, from Lemma~\ref{lem:shssujs7} we know that $\sigma'\leq c = \tfrac{d}{s}$. So, it suffices to show that
\[\left(\frac{\tau}{s}-\frac{\tau-1}{s-1}\right)\left(1-\frac{s}{d}\right)\sigma \leq 1 + \frac{(\tau-1)(\sigma-1)}{s-1}.\]
After straightforward simplification we observe that this inequality is equivalent to
$(s-\tau) + (\tau-2)\sigma + \tfrac{\sigma}{d}(s+\tau) \geq 0$,
which clearly holds.

\section{Proof of Lemma~\ref{lem:lem8}}
In the $s=1$ case the statement is trivially true. Indeed, we must have $\tau=1$ and thus $\Prob(\hat{S} =  \{1,2,\dots,d\})=1$, $h^{\hat{S}} = h$, and hence \[\E\left[ (h^{\hat{S}})^T \bQ h^{\hat{S}} \right] = h^T \bQ h.\] This finishes the proof since $\tfrac{\tau-1}{s_1}=0$.

Consider now the $s>1$ case. From Lemma~3 in \citet{RT:PCDM} we get
\begin{equation} \label{eq:9sjs8j3sj8}\E\left[(h^{\hat{S}})^T \bQ h^{\hat{S}}\right] = \sum_{i \in \hat{S}}\sum_{j \in \hat{S}} \bQ_{ij} h^i  h^j = \sum_{i=1}^d \sum_{j=1}^d p_{ij} \bQ_{ij} h^i  h^j,\end{equation}
where $p_{ij} = \Prob(i \in \hat{S} \;\& \; j \in \hat{S})$. One can easily verify that
\[p_{ij} = \begin{cases}\frac{\tau}{s}, & \text{if } i=j,\\
\frac{\tau(\tau-1)}{s(s-1)}, & \text{if } i \neq j \text{ and } i\in \B_{l}, \; j \in \B_l \text{ for some } l,\\
\frac{\tau^2}{s^2}, & \text{if } i\neq j \text{ and } i\in \B_{k}, \; j\in \B_{l} \text{ for }k \neq l.\end{cases}\]
In particular, the first case follows from Eq (32)  and the second from Eq (37) in \citet{RT:PCDM}. It only remains to substitute $p_{ij}$ into \eqref{eq:9sjs8j3sj8} and transform the result into the desired form.

\end{document}

%% file: commandsICML2014.tex
\theoremstyle{plain}
\newtheorem{theorem}{Theorem}

\newtheorem{lemma}[theorem]{Lemma}

\theoremstyle{definition}

\newcommand{\R}{\mathbb{R}}
\newcommand{\B}{{\cal P}}
\newcommand{\bA}{\mathbf{A}}
\newcommand{\bQ}{\mathbf{Q}}
\newcommand{\bG}{\mathbf{G}}
\newcommand{\bM}{\mathbf{M}}
\newcommand{\bU}{\mathbf{U}}

\DeclareMathOperator{\Diag}{Diag}

\newcommand{\TT}[0]{{\tau}}

\newcommand{\st}{\;:\;}
\newcommand{\ve}[2]{\langle #1 ,  #2 \rangle}

\newcommand{\eqdef}{:=}

\newcommand{\Prob}{\mathbf{Prob}}
\newcommand{\E}{\mathbf{E}}

\newcommand{\vt}[2]{#1_{#2}}

\newcommand{\mrow}[2]{#1_{{#2}{:}}}
\newcommand{\mcol}[2]{#1_{{:}{#2}}}

\newcommand{\U}{e}

